\documentclass[10pt, journal]{IEEEtran}  

\IEEEoverridecommandlockouts     

\usepackage{color,amsmath,amsthm,amstext,amsfonts,amssymb, mathrsfs,mathtools}
\usepackage{graphicx,algorithm}
\usepackage[algo2e,linesnumbered]{algorithm2e}
\usepackage{tikz}
\usepackage{setspace}
\usepackage{array} 
\usepackage{booktabs}  
\usepackage{bbm}
\usepackage{wasysym}
\usepackage{textcomp}
\usepackage{hyperref} 
\usepackage{caption}
\usepackage{subcaption}
\usepackage[sort,compress]{cite}

\newtheorem{lemma}{Lemma}

\newtheorem{theorem}{Theorem} 
\newtheorem{definition}{Definition}

\theoremstyle{remark}
\newtheorem{remark}{Remark}

\SetKwInput{KwInput}{Input}

\newcounter{l1}
\newcounter{l2}
\newcounter{l3}
\newcommand{\bdotlist}{\begin{list}{$\bullet$}{}}
\newcommand{\bboxlist}{\begin{list}{$\Box$}{}}
\newcommand{\bbboxlist}{\begin{list}{\raisebox{.005in}{{\tiny $\blacksquare$ \ \ }}}{}}
\newcommand{\bdashlist}{\begin{list}{$-$}{} }
\newcommand{\blist}{\begin{list}{}{} }
\newcommand{\barablist}{\begin{list}{\arabic{l1}}{\usecounter{l1}}}
\newcommand{\balphlist}{\begin{list}{(\alph{l2})}{\usecounter{l2}}}
\newcommand{\bAlphlist}{\begin{list}{\Alph{l2}.}{\usecounter{l2}}}
\newcommand{\bdiamlist}{\begin{list}{$\diamond$}{}}
\newcommand{\bromalist}{\begin{list}{(\roman{l3})}{\usecounter{l3}}}

\newcommand{\beq}{\begin{equation}}
\newcommand{\eeq}{\end{equation}}

\let\[\equation
\let\]\endequation

\newcommand{\tn}{\textnormal}

\DeclarePairedDelimiterX{\Norm}[1]{\lVert}{\rVert}{#1}

\DeclareMathOperator*{\argmin}{arg\,min}

\allowdisplaybreaks

\title{Online Learning for Incentive-Based Demand Response}

\author{Deepan Muthirayan, and Pramod P. Khargonekar
\thanks{This work is supported in part by the National Science Foundation under Grant ECCS-1839429.
D. Muthirayan and P. P. Khargonekar are with the Department of Electrical Engineering and Computer Sciences, University of California Irvine, Irvine, CA (emails: deepan.m@uci.edu, pramod.khargonekar@uci.edu).}
}

\begin{document}

\maketitle

\begin{abstract}
In this paper, we consider the problem of learning online to manage Demand Response (DR) resources. A typical DR mechanism requires the DR manager to assign a baseline to the participating consumer, where the baseline is an estimate of the counterfactual consumption of the consumer had it not been called to provide the DR service. A challenge in estimating baseline is the incentive the consumer has to inflate the baseline estimate. We consider the problem of learning online to estimate the baseline and to optimize the operating costs over a period of time under such incentives. 
We propose an online learning scheme that employs least-squares for estimation with a perturbation to the reward price (for the DR services or load curtailment) that is designed to balance the exploration and exploitation trade-off that arises with online learning. We show that, our proposed scheme is able to achieve a very low regret of $\mathcal{O}\left((\log{T})^2\right)$ with respect to the optimal operating cost over $T$ days of the DR program with full knowledge of the baseline, and is individually rational for the consumers to participate. Our scheme is significantly better than the averaging type approach, which only fetches $\mathcal{O}(T^{1/3})$ regret. 
\end{abstract}

\section{Introduction}

Demand Response (DR) programs \cite{Albadi2008} are potentially powerful tools to modulate the demand for electricity in a wide variety of situations. For example, at certain times such as mid-afternoons on hot summer days, the supply of additional electric power is scarce and expensive. At these times, it is more cost-effective to reduce demand than to increase supply to maintain power balance. Another scenario is a grid with high renewable penetration. Here, DR promises to be a better alternative compared to other expensive and polluting reserves to balance the variability in renewable generation. Realizing its potential, the 2005 Energy Policy Act provided the Congressional mandate to promote DR in organized wholesale electricity markets. The FERC order 745 \cite{federal2011demand} met this mandate by prescribing that demand response resource owners should be allowed to offer their demand reduction as if it were a supply resource rather than a bid to reduce demand so that the market operates fairly.

Dynamic pricing based DR programs 
can ideally achieve market efficiency, but they require complex metering and communication infrastructure to achieve this, which raises their implementation costs 
\cite{mathieu2013residential}. Furthermore, consumers may not be responsive to dynamic pricing \cite{faruqui2011dynamic}. Alternatively, consumers could be signaled to reduce consumption and paid for their load reductions. Such schemes are referred to as Incentive-based DR programs or Demand Reduction programs. There are two key components of any incentive-based DR program: (a) a baseline against which demand reduction is measured, (b) a payment scheme for agents who reduce their consumption from the baseline.

Thus, incentive-based DR programs require an established baseline against which consumer's load reduction is measured. The baseline is an estimate of the consumption when the consumer is not participating in the DR program. There are several ways to approach the estimation of baseline. One could, for example, use data to estimate the baseline. For example, the California Independent System Operator (CAISO) uses the average of the consumption on the ten most recent non-event days as the baseline estimate~\cite{CaisoDR2017}. 
Data driven approaches can be broadly classified as (a) averaging, (b) regression, and (c) control group methods. Typically, these methods are prone to baseline manipulation 
\cite{vuelvas2017rational}. There have have been reported cases in the past where the participants have artificially inflated their baseline for increasing payments \cite{gaming-examples}. Another class of approaches are based on {\it mechanism design}, where the consumers are elicited to report their baselines \cite{muthirayan2019mechanism, muthirayan2019minimal, satchidanandan2022two}. These approaches rely on suitably designed payment schemes to ensure that the manipulation or gaming of the reporting is minimal.

While many data-driven approaches for estimating baseline have been proposed in the literature \cite{coughlin2008estimating, grimm2008evaluating, mathieu2011quantifying, wijaya2014bias, nolan2015challenges, weng2015probabilistic, zhang2016cluster, nolan2015challenges, zhou2016forecast, hatton2016statistical, wang2018synchronous}, they typically consider the offline setting where sufficient data is available for estimation prior to the start of the DR program. The limitation is that these approaches cannot be used when the data is limited or when the underlying conditions can change. The mechanism design approaches can avoid the need for learning, but have limitations because the consumers can be unwilling to reveal their baselines or can even be unaware of their baselines. These considerations necessitate approaches that can learn online while running the DR program without needing the consumers to report any information. 



{\bf Contribution}: In this work, we consider the problem of managing DR resources where a participating consumer's baseline is to be estimated online, i.e., while running the DR program. We consider the setting where the DR program can only learn from the consumption data that it gathers over the course of time. 
The unique challenge of our setting is that the DR program manager has to simultaneously learn the consumers' baselines and optimize its operating costs with the information it gathers along the way. This makes this problem a typical online learning problem. 
Therefore, the exploration-exploitation trade-off in any online learning problem applies to our problem as well. The added complexity in a setting like ours is the incentive the consumers can have to interfere with the learning in order to inflate the baseline estimation. We formulate the online learning DR problem that incorporates all these characteristics and propose a  online learning DR scheme for this problem. We show that our method achieves $\mathcal{O}((\log{T})^2)$ regret with respect to the optimal operating cost over $T$ days of the DR program and is individually rational for the consumers to participate with an upfront payment for participation. Our main contribution is an online learning scheme for (incentive-based) DR that converges to the optimal cost with a very low regret and is at the same time individually rational. Ours is also the first work to formally study incentive-based DR as an online learning problem and present algorithms and regret guarantees.

\subsection{Related Works}

There exists substantial literature on baseline estimation methods \cite{coughlin2008estimating, grimm2008evaluating, mathieu2011quantifying, wijaya2014bias, nolan2015challenges, weng2015probabilistic, zhang2016cluster, nolan2015challenges, zhou2016forecast, hatton2016statistical, wang2018synchronous, muthirayan2019mechanism, muthirayan2019minimal}. These can be broadly classified into four classes: (a) averaging, (b) regression, (c) control group methods and (d) baseline reporting approaches.

\emph{Averaging methods} determine baselines by averaging the consumption on past days that are similar (\emph{e.g.,} in weather conditions) to the event
day. A detailed comparison of  different averaging methods is offered in \cite{coughlin2008estimating, grimm2008evaluating, wijaya2014bias}. Averaging
methods are simple but they suffer from estimation biases \cite{wijaya2014bias, weng2015probabilistic, nolan2015challenges}. \emph{Regression methods} estimate a load prediction model based on historical data which is then used to predict the baseline \cite{zhou2016forecast, mathieu2011quantifying}. They can potentially overcome biases incurred by averaging methods \cite{nolan2015challenges}. 
\emph{Control group methods} have been suggested to have better accuracy than averaging or regression type methods and do not require large amounts of historical data \cite{hatton2016statistical}. However, these methods require the SO to recruit an additional set of consumers and install additional metering infrastructure. In addition, prior data based analysis might be required to identify the most appropriate control group depending on the control group method deployed. This can raise the costs of implementation \cite{hatton2016statistical}. 


\emph{Baseline reporting approaches} were proposed in \cite{muthirayan2016mechanism, muthirayan2018baseline, muthirayan2019minimal, satchidanandan2022two} as an alternative baseline estimation method. These approaches employ the framework of mechanism design to design payment and selection schemes to ensure that the consumers report the correct baseline values. While these methods can provably reduce the baseline error from that of averaging type methods \cite{muthirayan2019mechanism}, they violate privacy and are infeasible when the consumers can be unaware of their baselines. 

In contrast to the above approaches, we propose an online approach that does not require large quantities of historical data, or a control group, or reporting private information.

{\bf Notation}: We denote the expectation over a probability distribution by $\mathbb{E}[\cdot]$. We use $\mathcal{O}(\cdot)$ for the standard big-O notation while $\widetilde{\mathcal{O}}(\cdot)$ denotes the big-O notation neglecting the poly-log terms. 
We denote the sequence $(x_{m_1}, x_{m_1+1}, \dots , x_{m_2})$ compactly by $x_{m_{1}:m_{2}}$ and the sequence $(x^{m_1}, x^{m_1+1}, \dots , x^{m_2})$ compactly by $x^{m_{1}:m_{2}}$.

\section{Demand Response Formulation}

We consider the problem of managing Demand Response (DR) in an online setting, where the consumers' utility functions are unknown to the System Operator (SO) and the SO has to learn the necessary consumer relevant parameters online. In a typical demand response program, the SO recruits consumers for demand response and calls them to provide load curtailment on certain days. To incentivize the consumers to curtail, the SO typically pays a certain price (reward/kWh) for the load curtailment the consumers provide. Therefore, the consumer's response depends on the incentive or the reward to reduce, which is the price set by the SO. In addition to the payment for the DR services, the SO incurs an additional cost for serving the final consumption after load curtailment. Thus, the total cost for the SO depends on the payments for the DR services and the cost to serve the final load after the curtailment.

Typically, the SO can only observe the final consumption and not the load curtailment. Therefore, in addition to the price (reward/kWh), the SO needs to specify a baseline consumption to quantify the load curtailment. Baselines are estimates of the power that would be consumed had the consumer not been called to provide load curtailment. The SO, typically, announces the mechanism to assign the consumer's baseline to the consumers participating in the DR program. Thus, the SO's DR policy is the procedure to set the price (reward/kWh) and the baseline. The objective of the System Operator is to choose a DR policy that minimizes its overall cost. 

We note that it is impossible for the SO to avoid under payment or over payment for the load curtailment without the knowledge of the consumer's correct baseline, which the SO need not know apriori. Here, we consider the setting, where the SO learns to set the correct baseline during the course of the DR program. The price and the baseline that the SO sets can vary from one day to the other and can be adapted depending on the response that the SO observes over its operation. The SO on any day has the following information: (i) the price for DR on all the previous days (ii) the baseline set for all the previous days and (iii) the final consumption of the consumers on all the previous days. The SO can use this information to set the price and the baseline consumption. Since the SO has to learn with the observations made on the fly and there is a cost that the SO incurs every day, this problem in effect is an online learning problem.

Like in a typical online learning problem, the SO has to balance the exploration and exploitation trade-off. Specifically, in this context, the SO cannot afford to set a constant price throughout to learn the correct baseline. This is because the total payment is a function of the assigned baseline, which creates an incentive for consumers to modify their consumption so as to inflate their future baselines and thence their future payments. Therefore, the SO has to be strategic in how the prices are set initially so as to infer the correct baseline over time. This is the exploration part. The exploration thus has to be balanced against deviating from the optimal outcome so as to ensure that on average the SO does not deviate from the best outcome. 
We characterize the effectiveness of the SO's DR policy, as is done typically in online learning, through a metric called {\it regret}, which in our case is the difference between the cumulative cost over a set of days and the achievable optimal cost with the full baseline information over the same set of days. Our objective is to develop an online learning scheme that achieves sub-linear regret and thereby achieve an outcome that on average converges to the best DR outcome. 

\subsection{DR Setting}

We index the days by $t$. We denote the number of consumers participating in the DR program by $N$. The SO, before any given day, decides whether to call a DR event or not. If it decides to call a DR event, it assigns a baseline $\hat{b}^i_t$ to the $i$th consumer participating in the DR program and the price for DR $p_t$ (reward/kWh) prior to day $t$. 
The consumers are paid at the price $p_t$ for the reduction of consumption from the assigned baseline. The price $p_t$ is a reward or incentive for the consumer to reduce its consumption. The price and the baseline is set by the SO using the following information: (i) the price for DR on all the previous days (ii) the baseline set for all the previous days and (iii) the final consumption of the consumers on all the previous days. Thus, the SO can adapt its price and the baselines online as it makes newer observations. As in any DR program, the SO announces the procedure for setting the price for DR and the baseline prior to the first day, which is the SO's DR policy. 

\subsection{Consumer Model}
\label{sec:con-mod}

We denote the electricity consumption of a consumer $i$ on day $t$ by $q^i_t$. Then, the payment $P$ to consumer $i$ for curtailing from $\hat{b}^i_t$ is given by
\beq
P^i_t = p_t(\hat{b}^i_t - q^i_t).  \nonumber 
\label{eq:reward}
\eeq

We denote the utility that the consumer derives from the electricity consumption $q^i_t$ by
\begin{equation}
u^i_t = u^i(q^i_t,\epsilon^i_t) = \left( a^i + \epsilon^i_t \right) q^i_t - \frac{d^i (q^i_t)^2 }{2},  \nonumber 
\label{eq:con-ut}
\end{equation}
where 
$\epsilon^i_t$ is a zero mean random variable and models the unpredictability or the uncertainty in the consumer's behavior. The assumption is that, by day $t$, the consumers observe their respective $\epsilon^i_t$s and that this information is private to them. 

As in a typical power market, the consumers pay a retail price to the electricity provider for their daily consumption. We denote the retail price that the consumers pay by $p_0$. Therefore, the net utility to consumer $i$ on day $t$ is given by
\beq
U^i_t(q^i_t) = u^i_t - p_0q^i_t + P^i_t. \nonumber 
\label{eq:netut-con}
\eeq

The correct average baseline $\tilde{b}^i$ for a consumer $i$ can be derived from the consumer's utility function. Following the definition that the correct baseline is the optimal consumption when the consumer is not called to provide DR, the correct average baseline for a consumer $i$ is given by
\beq
\tilde{b}^i =  \mathbb{E}_{\epsilon^i_t} b^i_t = \frac{a^i - p_0}{d^i}, ~~ \tn{where} ~ b^i_t = \frac{a^i + \epsilon^i_t - p_0}{d^i}. \nonumber 
\label{eq:av-bas}
\eeq 
The optimal consumption in the hypothetical case when the set baselines are fixed to the correct values and do not depend on the past consumption can be derived by minimizing $U^i_t$ individually. Therefore, the consumption for this hypothetical case is given by
\beq
s^i_t(p_t) = \argmin_{q^i_t} U^i_t(q^i_t) = \frac{a^i + \epsilon^i_t - p_0 - p_t}{d^i}.   
\label{eq:con-c}
\eeq



{\it Consumer's Optimal Decision}: In a DR setting, since a consumer's current consumption determines the future baseline and payments, the consumer typically has an incentive to modify its consumption to influence the future baselines and the DR payments. To model this effect, we consider the setting, where a consumer's current decision is also determined by its effect on the outcome of the next $m$ days. In this setting, the optimal consumer response on a day $t$ is given by
\beq
q^{*i}_t = \arg \max_{q^i_{t:t+m}} \mathbb{E}\sum_{s = t}^{t+m}U^i_s(q^i_s),  
\label{eq:optcondec}
\eeq
where expectation is over all randomness in $\epsilon^i_s, p_s, \hat{b}^i_s$ for all $s > t$. 

\subsection{System Operator's Objective}

The system operator's decision variables on a day $t$ are the price for DR and the baselines, which we collectively denote by $(p_t, \hat{b}^{1:N}_t)$. 
We denote the aggregate of the assigned baselines on a day $t$ by
\beq
\hat{b}_t = \sum_{i=1}^{N} \hat{b}^i_t. \nonumber 
\label{eq:aggbase}
\eeq
Similarly, we denote the aggregate consumption on a day $t$ by
\beq
q_t = \sum_{i=1}^{N} q^i_t. \nonumber 
\label{eq:aggcon}
\eeq

Typically, the SO has to procure power from an external market to serve the demand of the consumers. Therefore, the SO incurs a cost for procuring the power consumed by the consumers. We denote the cost of procuring an unit of power by $c$. Therefore, the total cost that is incurred by the SO on a day $t$ is the sum of the purchase cost and the cost for DR:
\beq
C_t(p_t, \hat{b}_t) = c q_t + p_t (\hat{b}_t - q_t). \nonumber
\label{eq:so-cost}
\eeq
Therefore, the SO's expected cost on day $t$ conditioned on the set baseline and the price is given by
\beq
\widetilde{C}_t(p_t, \hat{b}_t) = \mathbb{E}[C_t(p_t, \hat{b}_t) \vert p_t, \hat{b}_t] =  c \tilde{q}_t + p_t (\hat{b}_t - \tilde{q}_t), \nonumber 
\label{eq:so-cost}
\eeq
where $\tilde{q}_t  = \mathbb{E} q_t$. In the analysis, we assume that consumer chooses $q_t$ according to Eq. \eqref{eq:optcondec}. 

{\bf SO's Objective:} Let the price that minimizes SO's expected cost when the baselines are set to the correct values be denoted by $p^{*}$. Then, the consumption under this price, with the baselines set to the correct values, is given by $s^{*i}_t = s^i_t(p^{*})$ for all $i$. Therefore, the optimal expected cost for the SO, when the baselines are set to the correct values, is given by
\beq
\widetilde{C}^{*}_t  =  c \tilde{s}^{*}_t + p^{*} (\tilde{b} - \tilde{s}^{*}_t), \nonumber 
\label{eq:mincost} 
\eeq
where $\tilde{s}^{*}_t = \mathbb{E}_{\epsilon_t} s^{*}_t$ and $s^{*}_t =  \sum_{i = 1}^{N} s^{*i}_t$, $\tilde{b} =  \sum_{i = 1}^{N} \tilde{b}^{i}$. Since the primary objective of the SO is not to inflate the baseline and over pay, it is reasonable to define the regret with respect to the total optimal cost when the baselines are set to the correct values. Therefore, we define the SO's expected regret over a time period $T$, under a DR policy, as
\beq
R_T = \sum_{t = 1}^{T} \left( \mathbb{E} [\widetilde{C}_t(p_t, \hat{b}_t)]  - \widetilde{C}^{*}_t\right). \nonumber  
\label{eq:regret}
\eeq
The SO's objective is to prescribe a DR policy such that
\beq
\lim_{T \rightarrow \infty} \frac{R_T}{T} = 0 ~~ (\tn{No Regret}). \nonumber 
\label{eq:so-obj}
\eeq
The SO has to achieve zero regret on average while ensuring that the consumer's individual rationality is satisfied on average, i.e.,
\begin{align} 
& \lim_{T \rightarrow \infty} \frac{\mathbb{E}[\sum_{t = 1}^T U^i_t(q^i_t)] - TU^{*i}}{T} \geq 0~ \forall ~ i, \nonumber \\
& ~ (\tn{Individual Rationality}), \nonumber \\
& ~ \tn{where} ~ U^{*i} = \mathbb{E}[u^i(s^i_t(0),\epsilon^i_t) - p_0s^i_t(0)].   
\label{eq:con-indrat}
\end{align}

\begin{remark}[Individual Rationality]
The individual rationality condition in Eq. \eqref{eq:con-indrat} is essential, since, otherwise the SO can set a very low baseline and under pay the consumers for the DR services. Thus, this condition is essential in the formulation. Moreover, the consumer will not participate in the DR program if the consumer does not receive a benefit that is on average at least as much as the benefit when not participating in the DR program. Therefore, to enforce this constraint, we set $U^{*i}$ in the individual rationality condition as the optimal expected utility for a consumer when not participating in the DR program.
\end{remark}

\begin{remark}[Regret Definition]
The question is whether $\widetilde{C}^{*}_t$ is appropriate as the cost to be compared with in the regret. It can be shown that if the SO inflates the baseline by a certain quantity $\Delta b > 0$ then the optimal expected cost necessarily increases till the incentives for the consumers to participate in the DR program are positive. 
Therefore, given that the primary objective is to mitigate over payment while ensuring individual rationality, it follows that $\widetilde{C}^{*}_t$ is the right candidate for the cost to be compared with. 
\end{remark}

\section{Online Learning DR Mechanism}

In this section, we discuss our algorithm and present the properties of our algorithm formally. 

We recall that $q_1, q_2, q_3,....$ denote the sequence of consumption by the consumer and $p_1, p_2, . . . $ denote the sequence of price set by the SO for DR. The SO, at the end of a day $t$,  calculates a $\hat{b}^i_{1,t+1}$ and $\hat{b}^i_{t+1}$ for each consumer $i$ by
\begin{equation}
\left[ \begin{array}{c} \hat{b}^{e,i}_{1,t+1} \\ \hat{b}^{e,i}_{t+1} \end{array}\right] =\arg \min_{\hat{b},\hat{b}_1} \sum_{k=1}^t (q^i_k - (\hat{b} - \hat{b}_1 p_k))^{2}. 
\label{eq:ls-est-mulcon}
\end{equation}
The SO then assigns 
$\hat{b}^i_{t+1} = \hat{b}^{e,i}_{t+1}$ as the baseline for day $t+1$ to the $i$th consumer and calls all the consumers to provide DR service. The price $p_t$ for DR is given by 
\begin{equation}
p_t = p^* + \delta p e^{-t},
\label{eq:pr-dr-mul}
\end{equation}
where $\delta p$ is a constant. In addition, the SO also pays a payment $P_o$ to the consumer upfront. This payment is needed for meeting the individual rationality condition. Later, we give the specific form of this payment. 

\LinesNumberedHidden{\begin{algorithm}[]
\DontPrintSemicolon
\KwInput{$N, P^i_o$ for each $i \in [1,N]$}

Make the payment $P^i_o$ to each $i \in [1,N]$.

Announce the price sequence for the DR program as given by Eq. \eqref{eq:pr-dr-mul} and the process of baseline estimation.

{\bf Initialize} $\hat{b}^i_1$ for each $i \in [1,N]$ arbitrarily.

\For{$t = [1,T]$}{

Assign $\hat{b}^i_t$ as the baseline for each $i \in [1,N]$.

Set $p_t$ as the price for DR.

Receive the demand request $q^i_t$ from each consumer $i \in [1,N]$.

Serve $q^i_t$ to each consumer $i \in [1,N]$.

Incur the purchase cost $cq_t$ and the DR cost $p_t (\hat{b}_t - q_t)$.

Update $\hat{b}^i_t \rightarrow \hat{b}^i_{t+1}$ for each $i \in [1,N]$ according to Eq. \eqref{eq:ls-est-mulcon}.
}
\caption{Online Learning DR Mechanism (OL-DRM)}
\label{alg:olb}
\end{algorithm}}

\begin{remark}[Optimal Price for DR]
For the consumer utility model considered here \eqref{eq:con-ut}, it is straightforward to show that the optimal price $p^{*}$, given by the condition $\frac{d\widetilde{C}_t(p_t, \tilde{b})}{dp_t} = 0$, is $p^* = c/2$. 
\end{remark}


\begin{remark}[Exploration Strategy]
The prescribed policy for the SO explores by perturbing the price from the optimal price $p^{*}$. These perturbations cannot be persistent and therefore the prescribed policy reduces the perturbations with time. The decreasing of the perturbation is the balancing part of the exploration necessary to achieve sub-linear regret or ``No Regret".
\end{remark}


\begin{definition}
\[\widetilde{\Delta}^i_t = \frac{1}{d^i}\sum_{j = 1}^{m} \frac{p_{t+j}\left(- \sum_{s = 1}^{t+j-1} p_s p_t + \sum_{s=1}^{t+j-1} p_s^2\right)}{(t+j-1) \sum_{s = 1}^{t+j} \left(p_s - \frac{1}{t+j-1}\sum_{l=1}^{t+j-1} p_l\right)^2}. \nonumber \]
\[ P^i_o = \sum_{t=1}^T p_t\left( \frac{\sum_{k=1}^t p^{*} \delta p e^{-k} \widetilde{\Delta}^i_k}{\sum_{k=1}^t \left(p_k - \frac{\sum_{l=1}^t p_l}{t}\right)^2} \right). \nonumber \]
\label{def:payment}
\end{definition}
Next, we present the regret guarantee for our algorithm.
\begin{theorem}
Under the Algorithm OL-DRM (Algorithm \ref{alg:olb}), with $P^i_o$ given by Definition \ref{def:payment}, for $\delta p = \mathcal{O}(1)$, and $T > 1$,
\beq R_T = \mathcal{O}((\log{T})^2), \nonumber \eeq
and {\it individual rationality} holds for each $i \in [1,N]$. 
\label{thm:olb}
\end{theorem}
We give the detailed analysis in the next section. We observe that the regret guarantee for our approach leads to the desired ``No Regret" and individually rational outcome. 

\begin{remark}[Approach]
Our approach is the online equivalent of the regression approach to estimate baseline. An alternate approach is to not call the consumers for a certain number of days and use the consumption on these days to set the baseline for the future. This is the online equivalent of the averaging approach to estimate baseline. It can be shown that this approach leads to $\mathcal{O}(T^{1/3})$ regret. Our result, therefore, provides theoretical justification that regression type approaches can be superior to averaging type approaches.
\end{remark}


\begin{remark}[Extensions]
Our setting assumes that the consumer utility model is quadratic, and the consumer's decision depends only on a finite horizon $m$. Our approach can be trivially extended to the infinite horizon case, where the future benefits are discounted by a discounting factor. For this case, the algorithm requires no change except for the definition of $P^i_o$. We can use the same proof technique to analyse this case to show that the same regret is achievable. The extension to a general utility model is a subject of future work. 
\end{remark}

\section{Regret Analysis}

First, we derive an expression for $\hat{b}^i_t$. 
\begin{lemma}
Under the baseline estimation procedure of OL-DRM Algorithm \ref{alg:olb}, for any $t > 1$,
\beq \hat{b}^i_{t+1} =  \frac{ -\sum_{k = 1}^t p_k \sum_{k =1}^t p_k q^i_k + \sum_{k = 1}^{t} p^2_k \sum_{k=1}^t q^i_k }{t \sum_{k=1}^t \left(p_k - \frac{\sum_{l=1}^t p_l}{t}\right)^2}.\nonumber \eeq  
\label{lem:ls-exp}
\end{lemma}

\begin{proof}
Let 
\[ \Phi = \left[ \begin{array}{cc} \sum_{k = 1}^{t} p^2_k & -\sum_{k = 1}^t p_k \\  -\sum_{k = 1}^t p_k &  t \end{array}\right]. \nonumber \]
The determinant of matrix $\Phi$ is given by
\[ \tn{Det}(\Phi) = t \sum_{k=1}^t \left(p_k - \frac{\sum_{l=1}^t p_l}{t}\right)^2. \nonumber \]
Now, given how $p_t$ is defined (Eq. \eqref{eq:pr-dr-mul}), $\left(p_k - \frac{\sum_{l=1}^t p_l}{t}\right)^2 > 0$ for $k = 1$ and any $t > 1$. Therefore, $\Phi$ is invertible for any $t > 1$. Therefore, from standard least-squares estimation solution, it follows that
\[
\left[ \begin{array}{c} \hat{b}^{e,i}_{1,t+1} \\ \hat{b}^{e,i}_{t+1} \end{array}\right] = \Phi^{-1} \left[ \begin{array}{c} -\sum_{k =1}^t p_k q^i_k \\ \sum_{k=1}^t q^i_k \end{array}\right]. \nonumber \]
By using the standard formula for the inverse of a matrix, which for a given matrix $A$ is $\tn{Adj}(A)^\top/\tn{Det}(A)$, we get that, for any $t > 1$,
\begin{align}
& \left[ \begin{array}{c} \hat{b}^{e,i}_{1,t+1} \\ \hat{b}^{e,i}_{t+1} \end{array}\right] = \frac{\left[ \begin{array}{cc} t & \sum_{k = 1}^t p_k \\  \sum_{k = 1}^t p_k & \sum_{k = 1}^{t} p^2_k \end{array}\right]}{t \sum_{k=1}^t \left(p_k - \frac{\sum_{l=1}^t p_l}{t}\right)^2} \left[ \begin{array}{c} -\sum_{k =1}^t p_k q^i_k \\ \sum_{k=1}^t q^i_k \end{array}\right], \nonumber  \\
& \tn{i.e.}, ~ \hat{b}^{e,i}_{t+1} =  \frac{ -\sum_{k = 1}^t p_k \sum_{k =1}^t p_k q^i_k + \sum_{k = 1}^{t} p^2_k \sum_{k=1}^t q^i_k }{t \sum_{k=1}^t \left(p_k - \frac{\sum_{l=1}^t p_l}{t}\right)^2}. \nonumber 
\end{align}
\end{proof}

Next, we derive an expression for the optimal consumer response given by Eq. \eqref{eq:optcondec}. 
\begin{lemma}
The optimal consumer response under OL-DRM Algorithm \ref{alg:olb} is given by
\begin{align}
    & q^{*i}_t = b^i_t - \frac{p_t}{d^i} + \widetilde{\Delta}^i_t, ~~ \widetilde{\Delta}^i_t = \frac{1}{d^i}\sum_{j = 1}^{m} p_{t+j}\frac{\partial \hat{b}^i_{t+j}}{\partial q^i_t}, \nonumber \\
    & \frac{\partial \hat{b}^i_{t+j}}{\partial q^i_t} =  \frac{- \sum_{s = 1}^{t+j-1} p_s p_t + \sum_{s=1}^{t+j-1} p_s^2}{(t+j-1) \sum_{s = 1}^{t+j} \left(p_s - \frac{1}{t+j-1}\sum_{l=1}^{t+j-1} p_l\right)^2}. \nonumber 
\end{align} 
\label{lem:opt-con-resp}
\end{lemma}

\begin{proof}
Recall that the optimal consumer decision is given by
\begin{align}
q^{*i}_t = \arg \max_{q^i_{t:t+m}} \mathbb{E}_{\epsilon^i_{t+1:t+m}} [J(q^i_{t:t+m})], ~~  \nonumber  \\
J(q^i_{t:t+m}) = \left(\sum_{j = 1}^{m}U^i_{t+j}(q^i_{t+j})\right). \nonumber 
\end{align}
The general expectation in Eq. \eqref{eq:optcondec} reduces to the specific expectation in the equation above because (i) the price sequence for DR is set deterministically, and (ii) the fact that, at time $t$, the only randomness in the baselines to be assigned in the future, $\hat{b}^i_{t+j}$ for $j \in [1,m]$, is in the $q^i_k$s for $k \in [t+1,t+m]$; refer to Lemma \ref{lem:opt-con-resp} for the expression for $\hat{b}^i_{t}$. 

Next, we observe that the only term in $U^i_{t+j}(q^i_{t+j})$ that is a function of $q^i_t$ for all $j \in [1,m]$ is the assigned baseline $\hat{b}^i_{t+j}$; see Lemma \ref{lem:opt-con-resp} for the dependence of  $\hat{b}^i_{t+j}$ on $q^i_t$. Given that $J(q^i_{t:t+m})$ is concave in $q^i_{t+j}$s and $q^i_t$, it follows that $\mathbb{E}_{\epsilon^i_{t+1:t+m}} [J(q^i_{t:t+m})]$ is also concave in $q^i_t$. Therefore, given the concavity, by applying first order condition for optimality, we get that the optimal $q^{i*}_t$ satisfies
\[ a^i + \epsilon^i_t - d^i q^{*i}_t - (p_0 + p_t) + d^i \widetilde{\Delta}^i_t = 0. \nonumber \] 
The final expression follows from here. \end{proof}

\begin{remark}
The optimal consumer response has the following terms: (i) $b^i_t$, the standard response when not participating in DR, (ii) the second term is the reduction incentivized by the reward/kWh $p_t$, and (iii) the third term is the inflation in consumption that arises from the incentive to inflate future baseline assignments. 
\end{remark}

In the next lemma, we derive an upper bound for the regret in terms of the cumulative error in the baseline estimation.
\begin{lemma}
The regret under the OL-DRM Algorithm \ref{alg:olb} is
\begin{align}
& R_T = \sum_{t=1}^T (c-p_t)\widetilde{\Delta}_t + \sum_{t=1}^T \frac{(\delta p)^2 e^{-2t}}{d} \nonumber \\
& + \sum_{t=1}^Tp_t\mathbb{E}[\hat{b}_t - \tilde{b}] + P_o, \nonumber 
\end{align} 
where $1/d = \sum_{i=1}^N 1/d^i$, ~ $\widetilde{\Delta}_t = \sum_{i=1}^N  \widetilde{\Delta}^i_t, P_o = \sum_i P^i_o$.
\label{lem:reg-1}
\end{lemma}
\begin{proof}
In the following, for simplicity, we use $q^i_t$ as the notation for the consumption on day $t$ instead of $q^{*i}_t$. Let 
\[  C^{*i}_t = cs^{*i}_t + p^{*} (\tilde{b}^i - s^{*i}_t). \nonumber \]
Also, let
\[ C^{i}_t = cq^{i}_t + p_t (\hat{b}^i_t - q^{i}_t). \nonumber  \]
Then,
\begin{align}
& C^{i}_t -  C^{*i}_t = cq^{i}_t + p_t (\hat{b}^i_t - q^{i}_t) - cs^{*i}_t - p^{*} (\tilde{b}^i - s^{*i}_t) \nonumber \\
& = (c-p_t)q^i_t -(c-p^{*})s^{*i}_t + (p_t - p^{*})\tilde{b}^i + p_t(\hat{b}^i_t - \tilde{b}^i) \nonumber \\
& = c(q^i_t - s^{*i}_t) + p^{*}s^{*i}_t -p_tq^i_t + (p_t - p^{*})\tilde{b}^i + p_t(\hat{b}^i_t - \tilde{b}^i). \nonumber 
\end{align}
From Eq. \eqref{eq:con-c} and Lemma \ref{lem:opt-con-resp},
\[
s^{*i}_t  = b^i_t - \frac{p^{*}}{d^i}, ~
q^{i}_t = b^i_t - \frac{p_t}{d^i} + \widetilde{\Delta}^i_t. \nonumber \]
Therefore,
\begin{align}
& C^{i}_t -  C^{*i}_t = c \widetilde{\Delta}^i_t + \frac{c}{d_i}(p^{*} - p_t)  + p^{*}s^{*i}_t -p_tq^i_t \nonumber \\
& + (p_t - p^{*})\tilde{b}^i + p_t(\hat{b}^i_t - \tilde{b}^i) \nonumber \\
& =  c \widetilde{\Delta}^i_t + \frac{c}{d_i}(p^{*} - p_t) + (p_t - p^{*})\tilde{b}^i + p_t(\hat{b}^i_t - \tilde{b}^i) \nonumber \\
& + p^{*}\left( b^i_t - \frac{p^{*}}{d^i}\right) - p_t\left(b^i_t - \frac{p_t}{d^i} + \widetilde{\Delta}^i_t \right). \nonumber
\end{align}
Taking expectation on both sides
\begin{align}
& \mathbb{E}[C^{i}_t] -  \widetilde{C}^{*i}_t =  c \widetilde{\Delta}^i_t + \frac{c}{d_i}(p^{*} - p_t) + (p_t - p^{*})\tilde{b}^i \nonumber \\
& + p^{*}\left( \tilde{b}^i - \frac{p^{*}}{d^i}\right) - p_t\left(\tilde{b}^i - \frac{p_t}{d^i} + \widetilde{\Delta}^i_t \right) + \mathbb{E}[p_t(\hat{b}^i_t - \tilde{b}^i)]  \nonumber \\
& = (c-p_t)\widetilde{\Delta}^i_t + \frac{c}{d_i}(p^{*} - p_t) + \frac{1}{d^i}\left( p_t^2 - {p^{*}}^2\right) \nonumber \\
& + \mathbb{E}[p_t(\hat{b}^i_t - \tilde{b}^i)] \nonumber \\
& = (c-p_t)\widetilde{\Delta}^i_t + \frac{(\delta p)^2 e^{-2t}}{d^i} + \mathbb{E}[p_t(\hat{b}^i_t - \tilde{b}^i)]. \nonumber 
\end{align}
Therefore,
\[ \mathbb{E}[\widetilde{C}_t] -  \widetilde{C}^{*}_t = (c-p_t)\widetilde{\Delta}_t + \frac{(\delta p)^2 e^{-2t}}{d} + \mathbb{E}[p_t(\hat{b}_t - \tilde{b})]. \nonumber \] 
The final result follows from here.
\end{proof}

\begin{remark}
The regret has the following terms: (i) the first term reflects the increase in the power purchase cost from consumption inflation and the decrease in DR payments from consumption inflation, (ii) the second term reflects the exploration cost that arises from the deviation from the optimal price, (iii) the third term reflects the increase in DR payments from baseline inflation, and (iv) the final term is the total payment made to the consumers upfront.  
\end{remark}

Next, we bound a key term that contributes to the consumption inflation term. 
\[ 
\Delta_{t,k} := \frac{p_{t+1}}{d}\left[ \frac{-\sum_{s = 1}^{t}p_s p_{t-k} + \sum_{s = 1}^{t}p_s^2}{t \sum_{s=1}^t (p_s - \frac{\sum_{l=1}^t p_l}{t})^2}\right].
\label{eq:Delta-t}
\]
\begin{lemma}
Under the OL-DRM Algorithm \ref{alg:olb}, for any $t > 1$
\[ \Delta_{t,k} = \mathcal{O}\left( (p^{*}+ \delta p) t^{ -1} \right). \nonumber \]
\label{lem:delta-t}
\end{lemma}
\begin{proof}
From how $p_t$ is defined (Eq. \eqref{eq:pr-dr-mul}), we get that 
\begin{align}
& \sum_{s=1}^{t} \left(p_s^2 - p_s p_{t-k} \right) \leq \sum_{s = 1}^{t} \left((p^* + \delta p e^{-s})^2 \right. \nonumber \\
& \left. - (p^* + \delta p (e)^{-s}) p^*\right). \nonumber 
\end{align}
Simplifying the above expression, we get
\begin{align}
& \sum_{s=1}^{t} p_s^2 - p_s p_{t-k} \leq \sum_{s = 1}^{t} \left( (\delta p) p^* e^{-s} + \delta p^2 e^{-2s} \right) \nonumber\\
& = \mathcal{O}\left((\delta p) p^* + (\delta p)^2\right). \nonumber 
\end{align}

Next, we bound the denominator of $\Delta_{t,k}$. 
\begin{align}
& t \sum_{s=1}^{t} \left(p_s - \frac{\sum_{l=1}^t p_l}{t}\right)^2 = t \left(\sum_{s=1}^{s=t} p^2_s \right) - \left( \sum_{s=1}^t p_s\right)^2 \nonumber\\
& = t \sum_{s = 1}^{t} \left(p^* + \delta p e^{-s}\right)^2 - \left(\sum_{s=1}^t p_s \right)^2 \nonumber\\
& = t^2 {p^*}^2 + 2tp^*\delta p \sum_{s = 1}^{t} e^{-s} + t \sum_{s = 1}^{t} \left(\delta p\right)^2 e^{-2s} \nonumber\\
&   - \left(\sum_{s=1}^t p_s \right)^2. \nonumber 
\end{align}
The term $\left(\sum_{s=1}^t p_s \right)^2 $ can be expanded as
\begin{align} 
& \left(\sum_{s=1}^t p_s \right)^2  = t^2 {p^*}^2 + 2 t p^* \delta p \sum_{s = 1}^{t} \delta p e^{-s} + \left(\sum_{s = 1}^{t}\delta p e^{-s}\right)^2. \nonumber 
\end{align}
Combining the two, we get
\begin{align}
& t \sum_{s=1}^{t} \left(p_s - \frac{\sum_k p_k}{t}\right)^2  =   t \sum_{s = 1}^{t} \left(\delta p\right)^2 e^{-2s} - \left(\sum_{s = 1}^{t}\delta p e^{-s}\right)^2 \nonumber \\
& \geq \frac{t\left(\delta p\right)^2(e^{2t}-1)}{2e^{2t}} - \left(\delta p\right)^2\frac{(e^{2t}-1)}{e^{2t}} = \mathcal{O}\left((\delta p)^2t\right). \nonumber 
\label{eq:din-Dt}
\end{align}
Combining the upper bound for the numerator and the lower bound for denominator, we get the final result.
\end{proof}

The bound on the consumption inflation term $\widetilde{\Delta}_t$ follows from adding $m$ terms of the type bounded in Lemma \ref{lem:delta-t}; see Lemma \ref{lem:opt-con-resp}. In the next lemma, we derive an upper bound for the cumulative error in the baseline estimation and the payment $P_o$. 
\begin{lemma}
Under the OL-DRM Algorithm \ref{alg:olb}, for $T > 1$,
\[ \sum_{t=1}^Tp_t\mathbb{E}[\hat{b}_t - \tilde{b}] + P_o = \mathcal{O}\left((\log{T})^2\right) . \nonumber \]
\label{lem:error-basest}
\end{lemma}
\begin{proof}
Let $\delta b^i_{t+1} = \left( \frac{\sum_{k=1}^t p^{*} \delta p e^{-k} \widetilde{\Delta}^i_k}{\sum_{k=1}^t \left(p_k - \frac{\sum_{l=1}^t p_l}{t}\right)^2} \right).$
Then, note that $\sum_t p_t \delta b^i_t = P^i_o$. From Lemma \ref{lem:ls-exp}, we have that
\[
\hat{b}^i_{t+1} =  \frac{ -\sum_{k = 1}^t p_k \sum_{k =1}^t p_k q^i_k + \sum_{k = 1}^{t} p^2_k \sum_{k=1}^t q^i_k }{t \sum_{k=1}^t \left(p_k - \frac{\sum_{l=1}^t p_l}{t}\right)^2}. \nonumber \]
Therefore, taking expectation on both sides, we get
\begin{align}
& \mathbb{E}[\hat{b}^i_{t+1}] =  \frac{ -\sum_{k = 1}^t p_k \sum_{k =1}^t p_k \mathbb{E}[q^i_k] + \sum_{k = 1}^{t} p^2_k \sum_{k=1}^t \mathbb{E}[q^i_k] }{t \sum_{k=1}^t \left(p_k - \frac{\sum_{l=1}^t p_l}{t}\right)^2}. \nonumber 
\end{align}
Therefore,
\begin{align}
& \mathbb{E}[\hat{b}^i_{t+1}] = \frac{ -\sum_{k = 1}^t p_k \sum_{k =1}^t p_k \left(\tilde{b}^i - \frac{p_k}{d^i}+ \widetilde{\Delta}^i_k\right)}{t \sum_{k=1}^t \left(p_k - \frac{\sum_{l=1}^t p_l}{t}\right)^2} \nonumber \\
& + \frac{\sum_{k = 1}^{t} p^2_k \sum_{k=1}^t \left(\tilde{b}^i - \frac{p_k}{d^i} + \widetilde{\Delta}^i_k\right) }{t \sum_{k=1}^t \left(p_k - \frac{\sum_{l=1}^t p_l}{t}\right)^2}. \nonumber 
\end{align}
First, we observe that 
\[ \tilde{b}^i = \frac{-\sum_{k = 1}^t p_k \sum_{k =1}^t p_k \tilde{b}^i + t\tilde{b}^i\sum_{k = 1}^{t} p^2_k }{t \sum_{k=1}^t \left(p_k - \frac{\sum_{l=1}^t p_l}{t}\right)^2}. \nonumber \]
Therefore, it follows that
\begin{align}
& \mathbb{E}[\hat{b}^i_{t+1}] - \tilde{b}^i = \frac{ -\sum_{k = 1}^t p_k \sum_{k =1}^t p_k \widetilde{\Delta}^i_k + \sum_{k = 1}^{t} p^2_k \sum_{k=1}^t \widetilde{\Delta}^i_k}{t \sum_{k=1}^t \left(p_k - \frac{\sum_{l=1}^t p_l}{t}\right)^2}. \nonumber 
\end{align}
Therefore,
\begin{align}
& \mathbb{E}[\hat{b}_{t+1} - \tilde{b}] = \frac{ -\sum_{k = 1}^t p_k \sum_{k =1}^t p_k \widetilde{\Delta}_k + \sum_{k = 1}^{t} p^2_k \sum_{k=1}^t \widetilde{\Delta}_k}{t \sum_{k=1}^t \left(p_k - \frac{\sum_{l=1}^t p_l}{t}\right)^2}. \nonumber 
\end{align}
In the proof of Lemma \ref{lem:delta-t}, we showed that the denominator is lower bounded as 
\[ t \sum_{k=1}^t \left(p_k - \frac{\sum_{l=1}^t p_l}{t}\right)^2 \geq \mathcal{O}(t). \nonumber \]
Also, from the definition of $\widetilde{\Delta}_k$ and Lemma \ref{lem:delta-t}, it follows that
\[ \widetilde{\Delta}_k \leq \mathcal{O}\left(m k^{-1}\right). \nonumber \]
Now, we can simplify $\mathbb{E}[\hat{b}_{t+1} - \tilde{b}] $ further as
\begin{align}
& \mathbb{E}[\hat{b}_{t+1} - \tilde{b}] \nonumber \\
& \leq \frac{\left(\sum_{k = 1}^{t}  p^* \delta p e^{-k}  + \left( \delta p e^{-k}  \right)^2 \right) \sum_{k = 1}^{t}\widetilde{\Delta}_k}{\mathcal{O}(t)} - \delta b^i_{t+1} \nonumber \\
& = \mathcal{O}\left( \frac{\sum_{k = 1}^{t}\widetilde{\Delta}_k}{t}\right) - \delta b^i_t \leq \mathcal{O}\left( \frac{\log(T)}{t}\right) - \delta b^i_{t+1}. \nonumber 
\end{align}
Therefore,
\[ \sum_{t = 1}^T p_t\mathbb{E}[\hat{b}_{t} - \tilde{b}] + P_o \leq (\log{T})^2. \nonumber \]
Then, combining Lemma \ref{lem:reg-1} and Lemma \ref{lem:error-basest}, we get the final regret result. Next, we show that {\it individual rationality} is satisfied by our algorithm. 

{\bf Individual Rationality}: 

For convenience, we remove the superscript $i$ from all variables. Let $\mathbb{E}[\epsilon^2_t] = \sigma^2_t$. We consider $U^{*}$ first. 
\begin{align}
& \mathbb{E}[U^{*}] \nonumber \\
& = \mathbb{E}\left[(a+\epsilon_t)\left(\tilde{b}+\frac{\epsilon_t}{d}\right) - \frac{d}{2}\left(\tilde{b} + \frac{\epsilon_t}{d}\right)^2 - p_0\left( \tilde{b} + \frac{\epsilon_t}{d}\right) \right]\nonumber \\
& = a\tilde{b} + \frac{\sigma^2_t}{2d} -\frac{d\tilde{b}^2}{2} - p_0\tilde{b}. \nonumber 
\end{align}
From Lemma \ref{lem:opt-con-resp}, we recall that $q_t = \tilde{b} + \frac{\epsilon_t}{d} -\frac{p_t}{d} + \widetilde{\Delta}_t$. Hence,
\begin{align}
& \mathbb{E}[U_t(q_t)] = a\tilde{b} + a\left( \widetilde{\Delta}_t - \frac{p_t}{d}\right) -\frac{d}{2} \left( \tilde{b} + \widetilde{\Delta}_t - \frac{p_t}{d}\right)^2 - p_0\tilde{b} \nonumber \\
& -p_0\left(\widetilde{\Delta}_t - \frac{p_t}{d} \right) + p_t\left(\frac{p_t}{d} - \widetilde{\Delta}_t\right) + p_t\mathbb{E}[\hat{b}_t - \tilde{b}] +  \frac{\sigma^2_t}{2d}\nonumber \\
& = a\tilde{b} - p_0\tilde{b} -\frac{d\tilde{b}^2}{2} -\frac{d}{2}\left(\widetilde{\Delta}_t - \frac{p_t}{d}\right)^2 \nonumber \\
& + p_t\left(\frac{p_t}{d} - \widetilde{\Delta}_t\right) + p_t\mathbb{E}[\hat{b}_t - \tilde{b}] +  \frac{\sigma^2_t}{2d}. \nonumber 
\end{align}
Therefore,
\begin{align}
& \mathbb{E}[U_t(q_t)] - \mathbb{E}[U^{*}] = p_t\mathbb{E}[\hat{b}_t - \tilde{b}] + p_t\left(\frac{p_t}{d}- \widetilde{\Delta}_t\right)  \nonumber \\  
& -\frac{d}{2}\left(\widetilde{\Delta}_t - \frac{p_t}{d}\right)^2 = p_t\mathbb{E}[\hat{b}_t - \tilde{b}] + \frac{p^2_t}{2d} -\frac{d\widetilde{\Delta}^2_t}{2}. \nonumber 
\end{align}  

Now, following the proof of Lemma \ref{lem:error-basest},
\begin{align}
& \mathbb{E}[\hat{b}_{t+1} - \tilde{b}] \nonumber \\
& = \frac{ -\sum_{k = 1}^t p_k \sum_{k =1}^t p_k \widetilde{\Delta}_k + \sum_{k = 1}^{t} p^2_k \sum_{k=1}^t \widetilde{\Delta}_k}{t \sum_{k=1}^t \left(p_k - \frac{\sum_{l=1}^t p_l}{t}\right)^2}. \nonumber 
\end{align}

Expanding the second term,
\begin{align} 
& \sum_{k = 1}^{t} p^2_k \sum_{k=1}^t \widetilde{\Delta}_k = \sum_{k = 1}^{t} \left(p^{*} + \delta p e^{-k}\right)^2 \sum_{k=1}^t \widetilde{\Delta}_k \nonumber \\
& = t {p^{*}}^2\sum_{k=1}^t \widetilde{\Delta}_k + 2p^{*}\delta p \sum_{k = 1}^t e^{-k}\sum_{k=1}^t \widetilde{\Delta}_k \nonumber \\
& + (\delta p)^2\sum_{k = 1}^t e^{-2k}\sum_{k=1}^t \widetilde{\Delta}_k. \nonumber  
\end{align}

Similarly, expanding the first term,
\begin{align}
& -\sum_{k = 1}^t p_k \sum_{k =1}^t p_k \widetilde{\Delta}_k \nonumber \\
& = -\sum_{k = 1}^t (p^{*} + \delta p e^{-k}) \sum_{k =1}^t (p^{*} + \delta p e^{-k}) \widetilde{\Delta}_k \nonumber\\
& = -t{p^{*}}^2 \sum_{k =1}^t \widetilde{\Delta}_k -t p^{*} \sum_{k =1}^t \delta p e^{-k} \widetilde{\Delta}_k -\sum_{k =1}^t \delta p e^{-k} \sum_{k =1}^t p^{*} \widetilde{\Delta}_k \nonumber \\
& -\sum_{k =1}^t \delta p e^{-k} \sum_{k =1}^t \delta p e^{-k} \widetilde{\Delta}_k. \nonumber 
\end{align}
Therefore,
\begin{align}
& \mathbb{E}[\hat{b}_{t+1} - \tilde{b}] \nonumber \\
& = \frac{ p^{*}\delta p \sum_{k = 1}^t e^{-k}\sum_{k=1}^t \widetilde{\Delta}_k + (\delta p)^2\sum_{k = 1}^t e^{-2k}\sum_{k=1}^t \widetilde{\Delta}_k }{t \sum_{k=1}^t \left(p_k - \frac{\sum_{l=1}^t p_l}{t}\right)^2} \nonumber \\
& + \frac{-t p^{*} \sum_{k =1}^t \delta p e^{-k} \widetilde{\Delta}_k-(\delta p)^2 \sum_{k =1}^t e^{-k} \sum_{k =1}^t e^{-k}}{t \sum_{k=1}^t \left(p_k - \frac{\sum_{l=1}^t p_l}{t}\right)^2}. \nonumber 
\end{align}
Using the bound on $\widetilde{\Delta}_k$,
\begin{align}
& \mathbb{E}[\hat{b}_{t+1} - \tilde{b}] \nonumber \\
& = \frac{ p^{*}\delta p \sum_{k = 1}^t e^{-k}\sum_{k=1}^t \widetilde{\Delta}_k + (\delta p)^2\sum_{k = 1}^t e^{-2k}\sum_{k=1}^t \widetilde{\Delta}_k }{t \sum_{k=1}^t \left(p_k - \frac{\sum_{l=1}^t p_l}{t}\right)^2} \nonumber \\
& + \frac{-(\delta p)^2 \sum_{k =1}^t e^{-k} \sum_{k =1}^t e^{-k}}{t \sum_{k=1}^t \left(p_k - \frac{\sum_{l=1}^t p_l}{t}\right)^2} - \delta b_{t+1}  \nonumber \\
& \geq -\mathcal{O}\left(\frac{1}{t}\right) - \delta b_{t+1}. \nonumber 
\end{align}
As in the proof of Lemma \ref{lem:error-basest},
\[ \widetilde{\Delta}^2_k \leq \mathcal{O}\left(m^2 k^{-1}\right). \nonumber \]
Therefore, $\exists$ some $t_0$ and some $\delta p = \mathcal{O}(1)$ such that sufficiently large such that, for all $t' > t_0$
\begin{align}
& \sum_{t = 1}^{T}\mathbb{E}[U_t(q_t)] - \mathbb{E}[U^{*}] + P_o
\geq  \nonumber \\
& \sum_{t = 1}^{T}\left( -\mathcal{O}\left(\frac{1}{t}\right) + \frac{{p^{*}}^2}{2d} -\frac{d\widetilde{\Delta}^2_t}{2}\right)  \nonumber\\
&  \geq -\mathcal{O}\left(\log{T}\right) + \mathcal{O}\left(T\right) > 0. \nonumber 
\end{align} 
Individual rationality follows from here. \end{proof}

\section{Conclusion}
In this work, we study the DR problem where the participating consumers' baselines have to be estimated online. The online nature of our baseline learning problem results in an exploration-exploitation trade-off between learning the baseline and optimizing the overall operating cost simultaneously, with an added complexity that the consumers can have incentives to inflate the baseline estimate. We propose a novel, online learning DR scheme for this problem and show that our approach achieves a low regret of $\mathcal{O}((\log{T})^2)$ over $T$ days of the DR program with respect to the best DR outcome with full information of the baselines and ensures that participating is individually rational for each consumer. The utility of our approach lies in the fact all prior approaches either require large quantity data or the consumers to report their baselines, both of which could be infeasible. Our contribution is a low regret online learning DR scheme. 

\bibliographystyle{IEEEtran} 
\bibliography{Refs.bib}

\begin{thebibliography}{10}
\providecommand{\url}[1]{#1}
\csname url@samestyle\endcsname
\providecommand{\newblock}{\relax}
\providecommand{\bibinfo}[2]{#2}
\providecommand{\BIBentrySTDinterwordspacing}{\spaceskip=0pt\relax}
\providecommand{\BIBentryALTinterwordstretchfactor}{4}
\providecommand{\BIBentryALTinterwordspacing}{\spaceskip=\fontdimen2\font plus
\BIBentryALTinterwordstretchfactor\fontdimen3\font minus
  \fontdimen4\font\relax}
\providecommand{\BIBforeignlanguage}[2]{{%
\expandafter\ifx\csname l@#1\endcsname\relax
\typeout{** WARNING: IEEEtran.bst: No hyphenation pattern has been}%
\typeout{** loaded for the language `#1'. Using the pattern for}%
\typeout{** the default language instead.}%
\else
\language=\csname l@#1\endcsname
\fi
#2}}
\providecommand{\BIBdecl}{\relax}
\BIBdecl

\bibitem{Albadi2008}
M.~H. Albadi and E.~El-Saadany, ``A summary of demand response in electricity
  markets,'' \emph{Electric power systems research}, vol.~78, no.~11, pp.
  1989--1996, 2008.

\bibitem{federal2011demand}
F.~E.~R. Commission, ``Demand response compensation in organized wholesale
  energy markets,'' \emph{Final Rule Report}, 2011.

\bibitem{mathieu2013residential}
J.~L. Mathieu, T.~Haring, J.~O. Ledyard, and G.~Andersson, ``Residential demand
  response program design: Engineering and economic perspectives,'' in
  \emph{European Energy Market (EEM), 2013 10th International Conference on
  the}.\hskip 1em plus 0.5em minus 0.4em\relax IEEE, 2013, pp. 1--8.

\bibitem{faruqui2011dynamic}
A.~Faruqui and J.~Palmer, ``Dynamic pricing and its discontents,''
  \emph{Regulation}, vol.~34, no.~3, pp. 16 -- 22, 2011.

\bibitem{CaisoDR2017}
CAISO, \emph{Demand Response User Guide. Version 4.3}, California ISO, May
  2017.

\bibitem{vuelvas2017rational}
J.~Vuelvas and F.~Ruiz, ``Rational consumer decisions in a peak time rebate
  program,'' \emph{Electric Power Systems Research}, vol. 143, pp. 533--543,
  2017.

\bibitem{gaming-examples}
\BIBentryALTinterwordspacing
J.~Pierobon. (2013) Two {FERC} settlements illustrate attempts to `game' demand
  response programs. Last accessed 2019-3-30. [Online]. Available:
  \url{https://www.energycentral.com/c/ec/ferc-settlements-illustrate-attempts-game-demand-response-programs}
\BIBentrySTDinterwordspacing

\bibitem{muthirayan2019mechanism}
D.~Muthirayan, D.~Kalathil, K.~Poolla, and P.~Varaiya, ``Mechanism design for
  demand response programs,'' \emph{IEEE Transactions on Smart Grid}, vol.~11,
  no.~1, pp. 61--73, 2019.

\bibitem{muthirayan2019minimal}
D.~Muthirayan, E.~Baeyens, P.~Chakraborty, K.~Poolla, and P.~P. Khargonekar,
  ``A minimal incentive-based demand response program with self reported
  baseline mechanism,'' \emph{IEEE Transactions on Smart Grid}, vol.~11, no.~3,
  pp. 2195--2207, 2019.

\bibitem{satchidanandan2022two}
B.~Satchidanandan, M.~Roozbehani, and M.~A. Dahleh, ``A two-stage mechanism for
  demand response markets,'' \emph{arXiv preprint arXiv:2205.12236}, 2022.

\bibitem{coughlin2008estimating}
K.~Coughlin, M.~A. Piette, C.~Goldman, and S.~Kiliccote, ``Estimating demand
  response load impacts: Evaluation of baseline load models for non-residential
  buildings in california,'' \emph{Lawrence Berkeley National Laboratory},
  2008.

\bibitem{grimm2008evaluating}
C.~Grimm, ``Evaluating baselines for demand response programs,'' in \emph{AEIC
  Load Research Workshop}, 2008.

\bibitem{mathieu2011quantifying}
J.~L. Mathieu, P.~N. Price, S.~Kiliccote, and M.~A. Piette, ``Quantifying
  changes in building electricity use, with application to demand response,''
  \emph{IEEE Transactions on Smart Grid}, vol.~2, no.~3, pp. 507--518, 2011.

\bibitem{wijaya2014bias}
T.~K. Wijaya, M.~Vasirani, and K.~Aberer, ``When bias matters: An economic
  assessment of demand response baselines for residential customers,''
  \emph{IEEE Transactions on Smart Grid}, vol.~5, no.~4, pp. 1755--1763, 2014.

\bibitem{nolan2015challenges}
S.~Nolan and M.~O’Malley, ``Challenges and barriers to demand response
  deployment and evaluation,'' \emph{Applied Energy}, vol. 152, pp. 1--10,
  2015.

\bibitem{weng2015probabilistic}
Y.~Weng and R.~Rajagopal, ``Probabilistic baseline estimation via gaussian
  process,'' in \emph{IEEE Power \& Energy Society General Meeting}, 2015.

\bibitem{zhang2016cluster}
Y.~Zhang, W.~Chen, R.~Xu, and J.~Black, ``A cluster-based method for
  calculating baselines for residential loads,'' \emph{IEEE Transactions on
  smart grid}, vol.~7, no.~5, pp. 2368--2377, 2016.

\bibitem{zhou2016forecast}
X.~Zhou, N.~Yu, W.~Yao, and R.~Johnson, ``Forecast load impact from demand
  response resources,'' in \emph{IEEE Power and Energy Society General
  Meeting}, 2016.

\bibitem{hatton2016statistical}
L.~Hatton, P.~Charpentier, and E.~Matzner-L{\o}ber, ``Statistical estimation of
  the residential baseline,'' \emph{IEEE Transactions on Power Systems},
  vol.~31, no.~3, pp. 1752--1759, 2016.

\bibitem{wang2018synchronous}
F.~Wang, K.~Li, C.~Liu, Z.~Mi, M.~Shafie-Khah, and J.~P. Catal{\~a}o,
  ``Synchronous pattern matching principle-based residential demand response
  baseline estimation: Mechanism analysis and approach description,''
  \emph{IEEE Transactions on Smart Grid}, vol.~9, no.~6, pp. 6972--6985, 2018.

\bibitem{muthirayan2016mechanism}
D.~Muthirayan, D.~Kalathil, K.~Poolla, and P.~Varaiya, ``Mechanism design for
  self-reporting baselines in demand response,'' in \emph{2016 American Control
  Conference (ACC)}.\hskip 1em plus 0.5em minus 0.4em\relax IEEE, 2016, pp.
  1446--1451.

\bibitem{muthirayan2018baseline}
------, ``Baseline estimation and scheduling for demand response,'' pp.
  4857--4862, 2018.

\end{thebibliography}

\end{document}